\pdfoutput=1

\documentclass[11pt]{article}

\usepackage[preprint]{acl}

\usepackage{times}
\usepackage{latexsym}

\usepackage[T1]{fontenc}

\usepackage[utf8]{inputenc}

\usepackage{microtype}

\usepackage{inconsolata}

\usepackage{graphicx}

\usepackage{algorithm}
\usepackage{algorithmic}

\usepackage{fancyhdr}
\usepackage{marvosym}
\usepackage{amsmath}
\usepackage{amssymb}
\usepackage{mathtools}
\usepackage{amsthm}
\usepackage{cleveref}
\usepackage{multirow}
\usepackage{xcolor}

\usepackage{longtable, array, booktabs}

\usepackage{comment}

\newtheorem{theorem}{Theorem}[section]
\newtheorem{proposition}[theorem]{Proposition}

\newtheorem{property}[theorem]{Property}

\newcommand{\R}{\mathcal{R}}
\newcommand{\C}{\mathcal{C}}
\newcommand{\X}{\mathcal{X}}
\newcommand{\D}{\mathbb{C}}
\newcommand{\pro}{\mathbb{P}}
\newcommand{\K}{\mathbb{K}}
\newcommand{\s}{s}
\newcommand{\x}{x}
\newcommand{\y}{\mathbf{y}}
\newcommand{\g}{\mathbf{g}}

\title{Shapley Uncertainty in Natural Language Generation}

\author{Meilin Zhu \\
  ISCAS  \\\And
  Gaojie Jin\thanks{Corresponding to g.jin@exeter.ac.uk.} \\
  University of Exeter 
  \\\And
  Xiaowei Huang \\
  University of Liverpool
  \\\And
  Lijun Zhang \\
  ISCAS }

\begin{document}
\maketitle
\begin{abstract}
In question-answering tasks, determining when to trust the outputs is crucial to the alignment of large language models (LLMs). 
\citet{kuhn2022semantic} introduces semantic entropy as a measure of uncertainty, by incorporating linguistic invariances from the same meaning. 
It primarily relies on setting threshold to measure the level of semantic equivalence relation.
We propose a more nuanced framework that extends beyond such thresholding by developing a Shapley-based uncertainty metric that captures the continuous nature of semantic relationships.
We establish three fundamental properties that characterize valid uncertainty metrics and prove that our Shapley uncertainty satisfies these criteria.
Through extensive experiments, we demonstrate that our Shapley uncertainty more accurately predicts LLM performance in question-answering and other datasets, compared to similar baseline measures.
\end{abstract}

\section{Introduction}

Although advancements have been made in natural language generation (NLG) tasks such as question answering and abstractive summarization~\citep{brown2020language,hoffmann2022training,chowdhery2023palm}, the understanding of uncertainty in foundation models remains limited. 
The absence of uncertainty metrics in transformer-based systems 
poses a threat to the reliability of their generated 
contents as an information source. 
Designing and validating reliable uncertainty measures is crucial for the development of safer AI systems, as highlighted by \citet{hendrycks2021unsolved}.

\citet{kuhn2022semantic} introduces semantic entropy, a measure that captures the constancy of meanings across linguistic variations. 
The core approach is a semantic equivalence relation, $\pro (\s_i\Rightarrow \s_j |\x)\ge t$ and $\pro (\s_j\Rightarrow \s_i |\x)\ge t$, where $t$ is a threshold, $\s_i$ and $\s_j$ are two output sentences given input $\x$, and $\s_j\Rightarrow \s_i$ represents $\s_i$ can be inferred from $\s_j$.
The above inequalities hold represents two sentences $\s_i$ and $\s_j$ that mean same thing, i.e., they belong to a same cluster $\C$.
Semantic entropy is designed to measure the average level of uncertainty inherent to the generated clusters.

The semantic entropy simplifies correlational information between sentences by reducing the continuous probability $\pro (\s_j\Rightarrow \s_i |\x)$ to a binary value, resulting in imprecise uncertainty estimation. 
This simplification may fail to capture that incorrect answers contribute varying degrees of uncertainty based on their semantic correlation with the correct answer. 
Note that, in this work, sentence correlation means the bidirectional relationship captured by probabilities $\pro (\s_i\Rightarrow \s_j |\x)$ and $\pro (\s_j\Rightarrow \s_i |\x)$.

\begin{table*}[t]
\centering
\vspace{-2mm}
\resizebox{\textwidth}{!}{
\setlength\tabcolsep{2pt}
\begin{tabular}{@{}lccccccc@{}}
\specialrule{.1em}{.075em}{.075em} 
Answer && Likelihood  & Semantic likelihood && \multicolumn{3}{c}{Shapley uncertainty} \\
$\quad\;\s$ && \( \pro(\s | \x) \) & \( \sum_{\s \in \C} \pro(\s | \x) \) && Wolfgang Amadeus Mozart & William Shakespeare & Ludwig van Beethoven \\
\cline{0-0} \cline{3-4} \cline{6-8} 
Wolfgang Amadeus Mozart  && 0.5 & \multirow{2}{*}{0.9} && 1 & 1 & 0.5 \\
Mozart      && 0.4 &  && 1 & 1 & 0.5 \\
Ludwig van Beethoven  && 0.1 & 0.1 && 0.5 & 0.5 & 1  \\
\cline{0-0} \cline{3-4} \cline{6-8} 
Uncertainty && 0.94 & 0.33 && \multicolumn{3}{c}{0.40} \\
\specialrule{.1em}{.075em}{.075em} 
\end{tabular}
}
\caption{Answers to the question ``Who wrote `Queen of the Night aria'?" 
We think language model understands the question but there is a factual error of ``Ludwig van Beethoven'' (a composer but not the right one). 
Our method takes correlation matrix and likelihood into account and computes the Shapley uncertainty value as 0.40.
}
\label{tab:1}
\vspace{3mm}

\centering
\vspace{-2mm}
\resizebox{\textwidth}{!}{
\setlength\tabcolsep{3 pt}
\begin{tabular}{@{}lccccccc@{}}
\specialrule{.1em}{.075em}{.075em} 
Answer && Likelihood  & Semantic likelihood && \multicolumn{3}{c}{Shapley uncertainty} \\
$\quad\;\s$ && \( \pro(\s | \x) \) & \( \sum_{\s \in \C} \pro(\s | \x) \) && Wolfgang Amadeus Mozart & Mozart & Leonardo da Vinci \\
\cline{0-0} \cline{3-4} \cline{6-8} 
Wolfgang Amadeus Mozart  && 0.5 & \multirow{2}{*}{0.9} && 1 & 1 & 0  \\
Mozart       && 0.4 &  && 1 & 1 & 0 \\
Leonardo da Vinci    && 0.1 & 0.1 && 0 & 0 & 1 \\
\cline{0-0} \cline{3-4} \cline{6-8} 
Uncertainty && 0.94 & 0.33 && \multicolumn{3}{c}{0.51} \\
\specialrule{.1em}{.075em}{.075em} 
\end{tabular}
}
\caption{Answers to the question ``Who wrote `Queen of the Night aria'?" 
We think language model may not understand the question as there is an irrelevant answer of ``Leonardo da Vinci'' (not a composer).
Thus the uncertainty of the answers in this group should be higher. 
In \Cref{tab:1}, the Shapley uncertainty is computed at 0.40, but it rises to 0.51 in this context. 
This increase is attributable to our method's assessment of ``Leonardo da Vinci'' as an irrelevant answer, in contrast to ``Ludwig van Beethoven'', which is deemed correlated. 
It is noteworthy that the uncertainties calculated by other methods remain consistent across both scenarios, which do not consider correlation.
}
\label{tab:2}
\end{table*}

To this end, we suggest employing the correlation matrix of output sentences for a more refined computation of NLG uncertainty. 
A key benefit of this method is its comprehensive utilization of all inter-sentence correlational data, moving beyond the simple clustering approach.
As shown in \Cref{tab:1,tab:2}, when addressing the question, ``Who wrote the `Queen of the Night aria'?'', semantic entropy simply categorizes ``Leonardo da Vinci'', ``Ludwig van Beethoven'', and ``Wolfgang Amadeus Mozart'' as distinct answers and calculates the entropy accordingly. 
In contrast, our method incorporates the intricate correlations among these responses.
For instance, ``Beethoven'' bears a closer relationship 
with the correct answer, ``Mozart'', than ``da Vinci'' does, because  both Beethoven and Mozart were composers, and da Vinci was not. 
Consequently, our approach considers answers involving ``Beethoven'' to be less uncertain than those mentioning ``da Vinci''. 
This nuanced treatment of answer relationships 
enables 
a more refined assessment of uncertainty in language model outputs.

To summarize, the contributions of this work are as follows:
\begin{itemize}
    \item In \Cref{sec:method}, we develop the Shapley uncertainty metric, which comprehensively incorporates correlational information among output sentences to quantify uncertainty in NLG.
    \item In \Cref{sec:experiments}, we demonstrate our method's superiority through comprehensive empirical studies across diverse tasks: open- and closed-book question-answering, and machine translation, utilizing the TriviaQA~\citep{joshi2017triviaqa}, CoQA~\citep{reddy2019coqa}, and WMT 2014~\citep{bojar2014findings} datasets. 
    Following the experimental framework of \citet{kuhn2022semantic}, we evaluate over 20 dataset-model combinations, validating our approach's generalizability across various question-answering tasks and LLM architectures, including DeepSeek, LLaMA and Qwen.
\end{itemize}

\section{Preliminaries}
\label{sec:pre}


The total uncertainty inherent in a prediction can be quantified by the predictive entropy of the output distribution, representing the amount of information known about the output conditional on the input. 
This entropy reaches its zenith when the output provides minimal information, especially assigning equal probabilities to all potential outcomes. 
More formally, for an input $\x$, the predictive entropy is defined as the entropy of the output (random variable) $\y$ given $\x$, i.e.,
\begin{equation}
\label{eq:1}
    h(\y|\x) = -\int \pro(y|\x) \ln \pro(y|\x) dy.
\end{equation}
We can recognize two distinct types of uncertainty: aleatoric uncertainty, related to inherent variability in the underlying data distribution, and epistemic uncertainty \citep{kendall2017uncertainties}, which emerges from missing knowledge.
Epistemic uncertainty is typically quantified via mutual information between model output distribution and its latent variables (e.g., model parameters or hidden representations), offers valuable insights but poses significant estimation challenges in large-scale models due to the requisite specialized techniques and computational demands. 
Rather than deriving epistemic uncertainty from model variance, it can alternatively be directly inferred using a secondary model~\citep{lahlou2022deup}. 
Following \citet{kuhn2022semantic}, in this work, we also eschew mutual information in favor of leveraging off-the-peg foundational models. 
Similarly, whereas methods such as \citet{malinin2020uncertainty} employs model ensembles to approximate the integral in \eqref{eq:1}, we opt for sampling from the output distribution of a single model. 
Prior methods, by simulating an ensemble's behavior within a singular model framework, provide a scalable approach to estimating model uncertainty, particularly pertinent to NLG given the extensive size of these models.

\begin{figure*}[t!]
\includegraphics[width=1.
\textwidth]{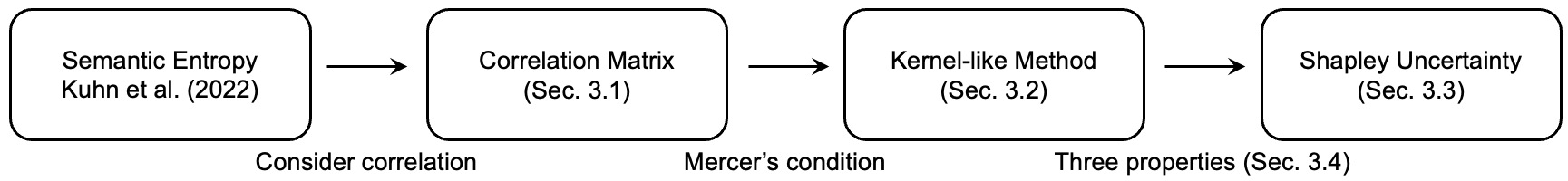}
\centering
\caption{
Technical insights, challenges, and solutions in obtaining the new Shapely uncertainty estimation method.
}
\label{fig:1}
\end{figure*}

\citet{kuhn2022semantic} introduces semantic entropy—an entropy which incorporates linguistic invariances created by shared meanings.
The core approach is a semantic equivalence relation, $\pro (\s_i\Rightarrow \s_j |\x)\ge t$ and $\pro (\s_j\Rightarrow \s_i |\x)\ge t$, where $t$ is a threshold, $\s_i$ and $\s_j$ are two output sentences given $\x$, and $\s_j\Rightarrow \s_i$ represents $\s_i$ can be inferred from $\s_j$.
The above inequalities hold represents two sentences $\s_i$ and $\s_j$ that mean same thing, i.e., they belong to a same cluster $\C$.
Then, \citet{kuhn2022semantic} defines the probability of $\C$ as
\begin{equation}
    \pro(\C|\x) = \sum_{\s\in \C} \pro (\s|\x),
\end{equation}
and the semantic entropy as
\begin{equation}
H(\C|\x)=-\sum_\C \pro(\C|\x) \ln \pro(\C|\x).
\end{equation}
They employ semantic entropy as a metric to assess the uncertainty of NLG outputs, thereby determining their reliability. 
Nonetheless, the methodology involves a rough clustering of different sentences and a simple approach to estimating semantic entropy. 
This results in a significant omission of correlational information among sentences, i.e., $\pro (\s_j\Rightarrow \s_i |\x)$ and $\pro (\s_i\Rightarrow \s_j |\x)$,  and consequently yields a potentially imprecise estimation of output uncertainty.


Recent advancements contrast traditional probabilistic methods by leveraging the language generation models themselves to estimate inherent uncertainty. 
For instance, \citet{lin2022teaching} has explored fine-tuning language models to articulate their confidence levels verbally. 
Concurrently, \citet{kadavath2022language} employs a strategy of sampling multiple generations and then responding to an NLG prompt to affirm the veracity of a suggested answer. 
Despite the promise shown by these techniques, they necessitate task-specific labels, additional training phases, and have demonstrated a propensity for unreliability when applied to out-of-distribution data.

\section{Shapley uncertainty}
\label{sec:method}

In this section, we introduce our novel NLG uncertainty estimation method, termed ``Shapley uncertainty''. 
This method is specifically designed to encapsulate a greater degree of correlational information among sentences, thereby enhancing the overall estimation of uncertainty.

As shown in \Cref{fig:1}, through a systematic steps, we develop the Shapley uncertainty metric which incorporates correlational information among sentences.
At first, following \citet{kuhn2022semantic}, we 
adopt a length-normalizing technique to ensure that the estimated uncertainty is independent of sentence length.
Some correlational information among sentences is deserted by the semantic entropy in \citet{kuhn2022semantic}, leading to a rough computation (i.e., underuse all information) of uncertainty.
In response to this limitation, our work introduces a method that incorporates the correlation matrix of NLG output. 
This approach is aimed at providing a more accurate depiction of NLG uncertainty.

\subsection{Correlation matrix}
\label{sec:m1}

The prior research of \citet{kuhn2022semantic} has introduced the semantic entropy, which clusters sentences based on the conditions $\pro (\s_i\Rightarrow \s_j |\x)\ge t$ and $\pro (\s_j\Rightarrow \s_i |\x)\ge t$ \citep{he2020deberta}. 
This method contributes to the enhancement of NLG uncertainty estimation by incorporating the semantic significance of different sentences. 
However, the usage of a threshold $t$ for clustering may result in the loss of valuable information, consequently leading to less accurate uncertainty estimations.



To address this issue, we introduce a placeholder semantic equivalence relation, $\D(\cdot,\cdot)$, which measures the ``correlation'' of any two sentences that mean the same thing,
$\D(\cdot,\cdot)$ is defined as follows:
\begin{equation}
\label{eq:generatecorrelationmatrix}
\begin{aligned}
\D(\s_i,\s_j|\x) &:= \frac{1}{2}\pro (\s_i\Rightarrow \s_j |\x)+\frac{1}{2}\pro (\s_j\Rightarrow \s_i |\x).
\end{aligned}
\end{equation}
Given the input $\x$, for $n$ output sentences $\s_1,...,\s_n$, we can generate a ``correlation matrix'' as follows
\begin{small}
\begin{equation}\nonumber
\begin{bmatrix}
\D(\s_1,\s_1)  & ...  & \D(\s_1,\s_{n}) \\
... & ... & ...  \\
\D(\s_{n},\s_1) & ... &  \D(\s_{n},\s_{n}) \\
\end{bmatrix}.
\end{equation}
\end{small}However, even with limited sentences, $\D(\cdot,\cdot)$ may not satisfy the Mercer's condition, i.e., the above ``correlation matrix" may not be positive semi-definite. 
This is mainly because $\pro (\s_i\Rightarrow \s_j |\x)$ is generated by a DNN-based method proposed by \citet{he2020deberta}. 
Notably, this method lacks an inherent mechanism to ensure that the outputs conform to Mercer's condition.

In statistics, a positive semi-definite correlation matrix indicates that the estimated correlations represent plausible relationships between sentences. 
However, if a matrix is not positive semi-definite, it might imply that the calculated correlations are not feasible, which can lead to issues in analyses that rely on correlation or covariance structures.
Thus, in the next subsection, we design a variant of the kernel function to make sure the generated correlation matrix is positive semi-definite.

\subsection{A variant of the kernel function}
\label{sec:m2}

The DNN-based method to calculate $\pro (\s_i\Rightarrow \s_j |\x)$ may not generate a normative correlation metric between sentences, i.e., the correlation operation $\D(\cdot,\cdot)$ based on $\pro (\s_i\Rightarrow \s_j |\x)$ may not satisfy the Mercer’s condition.
To address this issue, we design a variant of the kernel function as follows,
\begin{equation}
\label{eq:kernel}
\K(\D(\s_i,\s_j))=
\begin{cases}
    1, & \text{if $i=j$}\\
    \beta\cdot\kappa(1-\D(\s_i,\s_j)), & \text{otherwise}
\end{cases}
\end{equation}
where $\beta\in (0,1]$ is a hyper-parameter and $\kappa(\cdot)$ is a variant of kernel function, e.g., for a standard Gaussian kernel, we let $$\kappa(1-\D(\s_i,\s_j))=e^{-\frac{(1-\D(\s_i,\s_j))^2}{2}}.$$ 
In the experiments, we adopt the above Gaussian kernel and set $\beta=0.5$ as default, which is chosen through ablation studies given in \Cref{app:hyparam}.

The method in \eqref{eq:kernel} brings two benefits: Firstly, it introduces a non-linear kernel function to fix the correlation generated by DNN-based $\D(\cdot,\cdot)$,
enabling the algorithm to learn complex patterns between sentences and output a practicable metric to reflect the uncertainty.
Secondly, for finite sentences, the correlation matrix generated by \eqref{eq:kernel} can be proved to be positive semi-definite.

\begin{proposition}
Given a finite $n\in \mathbb{Z}^{+}$, there must exist $\beta \in \mathbb{R}^+$ such that the correlation matrix generated by $\K(\D(\cdot,\cdot))$ is positive semi-definite.
\end{proposition}
\begin{proof}
    See \Cref{app:a}. 
\end{proof}

\subsection{Shapley uncertainty}
\label{sec:m4}

Given a collection of $n$ sentences $\s_1,...,\s_n$, we compute their correlation matrix $\R\in\mathbb{R}^{n\times n}$ using \eqref{eq:kernel}, where $\R_{ij}=\K(\D(\s_i,\s_j))$. 
Suppose there are $n$ random variables ${\bf s}_1,...,{\bf s}_n$ corresponding to $\s_1,...,\s_n$, let $\tilde {\bf s}=[{\bf s}_1,...,{\bf s}_n]$ denote the multivariate random variable characterized by the correlation  matrix $\R$. 
The overall uncertainty of $\tilde {\bf s}$ can then be quantified through its differential entropy\footnote{We adopt a commonly used setting: suppose $\tilde {\bf s}$ is a multivariate Gaussian with correlation (covariance) matrix $\R$ and compute its entropy.}:
\begin{equation}
\label{eq:dfentropy}
    h(\tilde {\bf s}) = -\int \pro(\tilde s) \ln \pro(\tilde s) d\tilde s.
\end{equation}
\noindent
\emph{Remark. Having constructed the correlation matrix $\R$ for sentences $\s_1,...,\s_n$, we map these sentences to an $n$-dimensional multivariate random variable $\tilde {\bf s}$ characterized by $\R$. 
This mapping enables us to harness the complete correlational information encoded in $\R$ for computing the uncertainty of the sentence set through entropy measures.   
}



The effectiveness of differential entropy $h(\tilde {\bf s})$ as an uncertainty metric is limited in high-dimensional spaces. This limitation arises because a small subset of highly correlated dimensions can drive the differential entropy to negative infinity, masking the uncertainty contributions from other dimensions. 
For example, the differential entropy of a multivariate Gaussian mainly depends on its logarithmic determinant of covariance matrix, when any two dimensions become highly correlated, the determinant approaches zero, forcing the differential entropy toward negative infinity regardless of the behavior in other dimensions.

To address this, we propose an alternative uncertainty metric --- \textbf{Shapley uncertainty} $\phi(\tilde {\bf s})$ --- that decomposes into elementary uncertainties $\phi({\bf s}_i|\tilde {\bf s})$ for each dimension.
Leveraging Shapley method, we define the elementary uncertainty of ${\bf s}_i$ with respect to $\tilde {\bf s}$ as
\begin{equation}
\begin{aligned}
\label{eq:marginal uncertainty}
\phi({\bf s}_i|\tilde {\bf s}) & = \!\! \sum_{\X\subseteq\{1,...,n\}\backslash {i}} \!\!\! \frac{|\X|!(n\!-\!|\X|\!-\!1)!}{n!}\times \\
&\quad\Big(h\big([{\bf s}_j]_{j\in \X}\cup {\bf s}_i\big)-h\big(([{\bf s}_j]_{j\in \X}\big)\Big),
\end{aligned}
\end{equation}
and the whole uncertainty of $\tilde {\bf s}$ is defined as
\begin{equation}
\label{eq:shapley uncertainty}
    \phi(\tilde {\bf s})= 
    \sum_{\s_i\in \tilde {\bf s}} \phi({\bf s}_i|\tilde {\bf s}).
\end{equation}
\emph{Remark. We quantify individual uncertainty contributions through the Shapley value method, where the total uncertainty is the sum of these elementary uncertainties. 
In \eqref{eq:marginal uncertainty}, $|\X|$ denotes the set cardinality. 
The differential entropy $h(\cdot)$ is computed using the determinant of the correlation matrix, following the standard Gaussian setting outlined in the previous footnote.
}

\subsection{Three uncertainty metric properties}
\label{sec:m3}

We establish three essential properties that an eligible uncertainty metric must satisfy.


The first desirable property is the minimal uncertainty. 
Consider $\tilde {\bf s}=[{\bf s}_1,...,{\bf s}_n]$ and let $\g$ be a random variable corresponding to the sentence that shares the same semantic meaning of the element contributing minimal uncertainty to $\phi(\tilde {\bf s})$.
That is, $\g$ is totally correlated with $\arg\min_{{\bf s}_i}\phi({\bf s}_i|\tilde {\bf s})$.
Then, for all ${\bf s}_i\in \tilde {\bf s}$, the whole uncertainty of $\tilde {\bf s}\backslash {\bf s}_i \cup {\g}$ (${\bf s}_i$ is replaced by $\g$) is less than or equals to that of the original $\tilde {\bf s}$.
We formally define the minimal uncertainty property in the following.

\begin{property}[\bf Minimal uncertainty]
\label{property1}
Consider $\tilde {\bf s}=[{\bf s}_1,...,{\bf s}_n]$ and let $\g$ is  totally correlated with the element contributing minimal uncertainty to $\phi(\tilde {\bf s})$, formally expressed as:
$$\rho(\g, \underset{{\bf s}_j\in \tilde {\bf s}}{\arg\min}\phi({\bf s}_j|\tilde {\bf s}))=1,$$
where $\rho(\cdot,\cdot)$ represents the correlation.
Then, for any ${\bf s}_i\in \tilde {\bf s}$, replacing ${\bf s}_i$ with $\g$ should not increase the total uncertainty:
$$\phi(\tilde {\bf s}\backslash{\bf s}_i \cup {\g}) \le \phi(\tilde {\bf s}).$$
Moreover, if ${\bf s}_i \ne \arg\min_{{\bf s}_j}\phi({\bf s}_j|\tilde {\bf s})$, the inequality becomes strict:
$$\phi(\tilde {\bf s}\backslash{\bf s}_i \cup {\g}) < \phi(\tilde {\bf s}).$$
\end{property}

\begin{algorithm*}[t]
\caption{Shapley uncertainty}
\label{alg}
\hspace*{0.1in}\textbf{Input:} LLM, input $\x$, number of samples $n$. \\
\hspace*{0.1in}\textbf{Output:} Shapley uncertainty. \\
\hspace*{0.1in}1: Initialize a multiset of answers $\mathcal{O}\gets \emptyset$ \\
\hspace*{0.1in}2: \textbf{for} $i=1$ to $n$ \textbf{do} \hspace*{3.8in} $\triangleright$ Sample $n$ answers \\
\hspace*{0.1in}3:\hspace*{0.2in} Add $s_i=\text{LLM}(\x)$ to $\mathcal{O}$ \\
\hspace*{0.1in}4: \textbf{end for} \\
\hspace*{0.1in}5: \textbf{for} $i=1$ to $n-1$ \textbf{do} \hspace*{1.8in} $\triangleright$ Compute probability through \citet{he2020deberta} \\
\hspace*{0.1in}6:\hspace*{0.2in} \textbf{for} $j=i+1$ to $n$ \textbf{do} \\
\hspace*{0.1in}7:\hspace*{0.4in} Compute $\pro (\s_j\Rightarrow \s_i |\x)$ and $\pro (\s_i\Rightarrow \s_j |\x)$, $\s_i\in\mathcal{O}$, $\s_j\in\mathcal{O}$ \\
\hspace*{0.1in}8:\hspace*{0.2in} \textbf{end for} \\
\hspace*{0.1in}9: \textbf{end for} \\
\hspace*{0.1in}10: Compute all $\mathbb{C}(s_i,s_j|x)$ through $\pro (\s_i\Rightarrow \s_j |\x)$ and $\pro (\s_j\Rightarrow \s_i |\x)$ \hspace*{1.3in} $\triangleright$ As in \eqref{eq:generatecorrelationmatrix}\\
\hspace*{0.1in}11: Generate the correlation matrix $\R$, where $\R_{ij}=\K(\D(\s_i,\s_j))$  \hspace*{1.5in} $\triangleright$ As in \eqref{eq:kernel}\\
\hspace*{0.1in}12: Compute Shapley uncertainty \hspace*{3.in} $\triangleright$ As in \eqref{eq:marginal uncertainty} and \eqref{eq:shapley uncertainty} 
\end{algorithm*}

\noindent
\emph{Remark. 
The minimal uncertainty property reflects a fundamental intuition about uncertainty measurement: if we replace any sentence in the output set with a semantic duplicate of the least uncertain sentence, the total uncertainty should not increase. 
This aligns with the natural expectation that introducing redundant information cannot amplify the overall uncertainty of the system.
}

The second desirable property is the maximal uncertainty. 
Let $\g$ be a random variable corresponding to the sentence whose semantic meaning is entirely distinct from all others.
Then, for all ${\bf s}_i\in \tilde {\bf s}$, the whole uncertainty of $\tilde {\bf s}\backslash{\bf s}_i \cup {\g}$ is larger than or equals to that of the original $\tilde {\bf s}$.
We formalize this maximal uncertainty property in the definition below.

\begin{property}[\bf Maximal uncertainty]
\label{property2}
Let $\g$ correspond to a totally new sentence, i.e., $\g$ satisfies 
$$\forall {\bf s}_i\in \tilde {\bf s}, \quad \rho({\bf s}_i,\g)=0.$$ 
Then, replacing any variable ${\bf s}_i\in \tilde {\bf s}$ with $\g$ should not decrease the total uncertainty:
$$\phi(\tilde {\bf s}\backslash{\bf s}_i \cup {\g}) \geq \phi(\tilde {\bf s}).$$
Moreover, if ${\bf s}_i$ exhibits any correlation with others in $\tilde {\bf s}$, the inequality becomes strict:
$$\phi(\tilde {\bf s}\backslash{\bf s}_i \cup {\g}) > \phi(\tilde {\bf s}).$$
\end{property}

\noindent
\emph{Remark.
The maximal uncertainty property captures a complementary intuition: replacing any existing sentence with one that is semantically independent of the output set should not decrease total uncertainty. 
This reflects the natural principle that introducing novel, unrelated information cannot reduce the overall uncertainty of the system.
}

The third property is consistency. 
Let $\g$ correspond to the sentence that contributes higher uncertainty than ${\bf s}_j\in \tilde {\bf s}$ when combined with any subset of remaining ${\bf s}_i\in \tilde {\bf s}$, $i\ne j$. 
Then replacing ${\bf s}_j$ with $\g$ must increase the total uncertainty.

\begin{property}[\bf Consistency]
\label{property3}
Let $\g$ be a variable that may be correlated with elements in $\tilde {\bf s}$. 
For some ${\bf s}_j \in \tilde {\bf s}$, if the following inequality holds for all non-empty subsets $\X\subseteq\{1,...,n\}\backslash {j}$ with corresponding sequence ${\bf x}=[{\bf s}_i]_{i\in \X}$:
$$
\phi(\g|{\bf x}\cup {\g})>\phi({\bf s}_j|{\bf x}\cup{{\bf s}_j}),
$$
then the uncertainty metric must satisfy:
$$
\phi(\tilde {\bf s}\backslash {{\bf s}_j} \cup {\g})>\phi(\tilde {\bf s}).
$$
\end{property}

\noindent
\emph{Remark.
The consistency property establishes that if replacing a sentence increases its marginal uncertainty with respect to all possible combinations of other sentences, then the total uncertainty of the entire set must also increase. This ensures monotonic behavior in the uncertainty metric across local and global scales.
}

Since the differential entropy of a multivariate random variable depends on the determinant of its covariance matrix, using differential entropy as an uncertainty measure may not satisfy the the above properties.

\begin{proposition}
The metric of differential entropy over the whole multivariate distribution, as defined in \eqref{eq:dfentropy}, does not satisfy Properties \ref{property1} and \ref{property2}.
\end{proposition}

A straightforward example illustrates the above proposition: assuming the existence of ${\bf s}_i, {\bf s}_j \in \tilde {\bf s}$ with $\rho({\bf s}_i, {\bf s}_j)=1$, it is clear that the metric of differential entropy does not satisfy Properties \ref{property1} and \ref{property2} in this scenario.




\begin{proposition}
The metric of Shapley uncertainty, as defined in \eqref{eq:marginal uncertainty} and \eqref{eq:shapley uncertainty}, satisfies Properties~\ref{property1}, \ref{property2}, and \ref{property3}.
\end{proposition}
\begin{proof}
    See Appendix A.
\end{proof}

The central inspiration behind Shapley uncertainty is to tackle the semantic correlation intrinsic to natural language by converting the problem of uncertainty estimation into a correlated meaning-space.

As shown in \Cref{alg}, we provide the process of computing Shapley uncertainty in question-answering problems.
In the next section, following the setting of \citet{kuhn2022semantic} and \citet{liu2024uncertainty}, we empirically demonstrate the effectiveness of our metric to estimate the uncertainty of NLG outputs.

\begin{figure*}[t!]
\includegraphics[width=1
\textwidth]{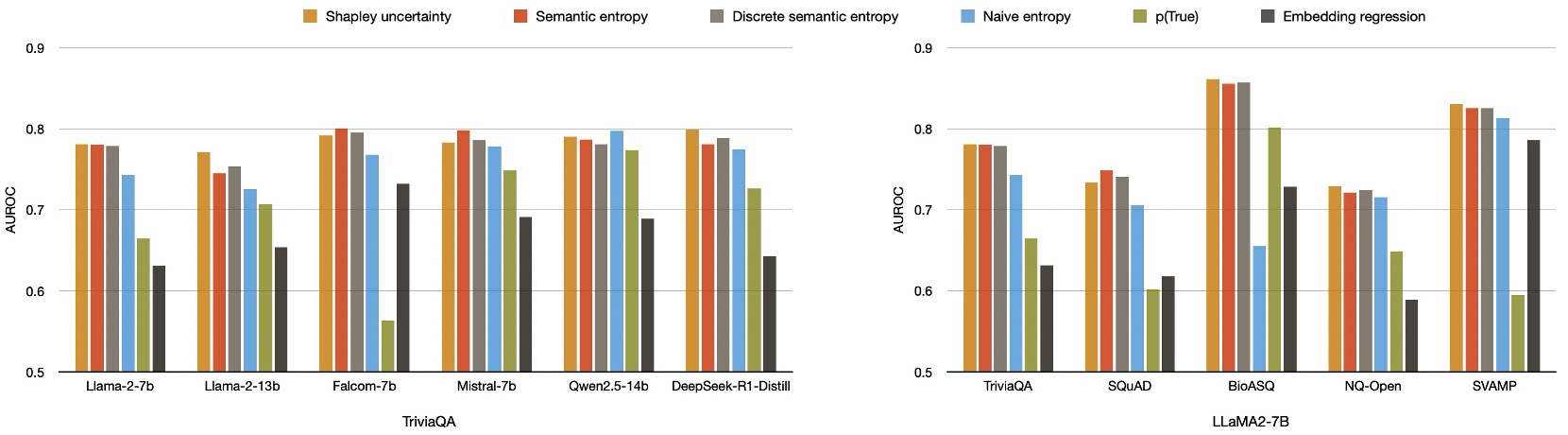}
\centering
\caption{
Cross-domain and LLM architectural generalization of our approach. Our Shapley uncertainty (blue) also surpassed baseline methods on TriviaQA (\textbf{Left}) and LLaMA2-7B (\textbf{Right}). 
}
\label{fig:2.5}
\vspace{-3mm}
\end{figure*}

\section{Empirical results}
\label{sec:experiments}

In this section, we empirically demonstrate the utility of Shapley uncertainty as an effective quantifier for uncertainty in NLG on free-form question-answering and machine translation tasks. 
Reliable uncertainty measures are essential as they inform us about the trustworthiness of model outputs, and an ideal measure should provide insight into the likelihood of a prediction being correct.

Prior studies, like \citet{filos2019benchmarking,kuhn2022semantic}, have approached uncertainty estimation by determining whether to trust a model's output for a given context. 
The receiver operator characteristic curve—AUROC—is a robust measure in this scenario, representing the probability that a randomly chosen correct answer is more uncertain than a randomly selected incorrect answer, with perfect uncertainty scoring 1 and a random (worst) measure scoring 0.5.




For \textbf{datasets}, CoQA, TriviaQA and WMT-14~\citep{bojar2014findings} are chosen to cover both open-book conversational, closed-book question-answer problems, and machine translation problems. 
We utilize a subset of 8000 questions from each above, aligning with the size of CoQA's training split. 
We also consider SOTA question-answering datasets like BioASQ~\citep{bioasq}, SQuAD~\citep{rajpurkar-etal-2016-squad}, SVAMP~\citep{patel-etal-2021-nlp} and NQ-Open~\citep{lee-etal-2019-latent} to evaluate cross-domain generalization performance of our approach.
Our evaluation metric of choice is a fuzzy matching criterion: $RougeL(\s,\s_{\text{true}})>0.3$, implying that an answer is considered correct if its Rouge-L \citep{lin2004automatic} similarity to the reference answer exceeds 0.3.
We employ the BLEU score \citep{papineni2002bleu,lin2004orange} as our scoring function $BLEU(\cdot, \cdot)$. 
Generated translations $\s$ are classified as correct if $BLEU(\s, \s_{\text{true}})$ exceeds 0.3, and incorrect otherwise.

\begin{figure}[t!]
\includegraphics[width=0.48
\textwidth]{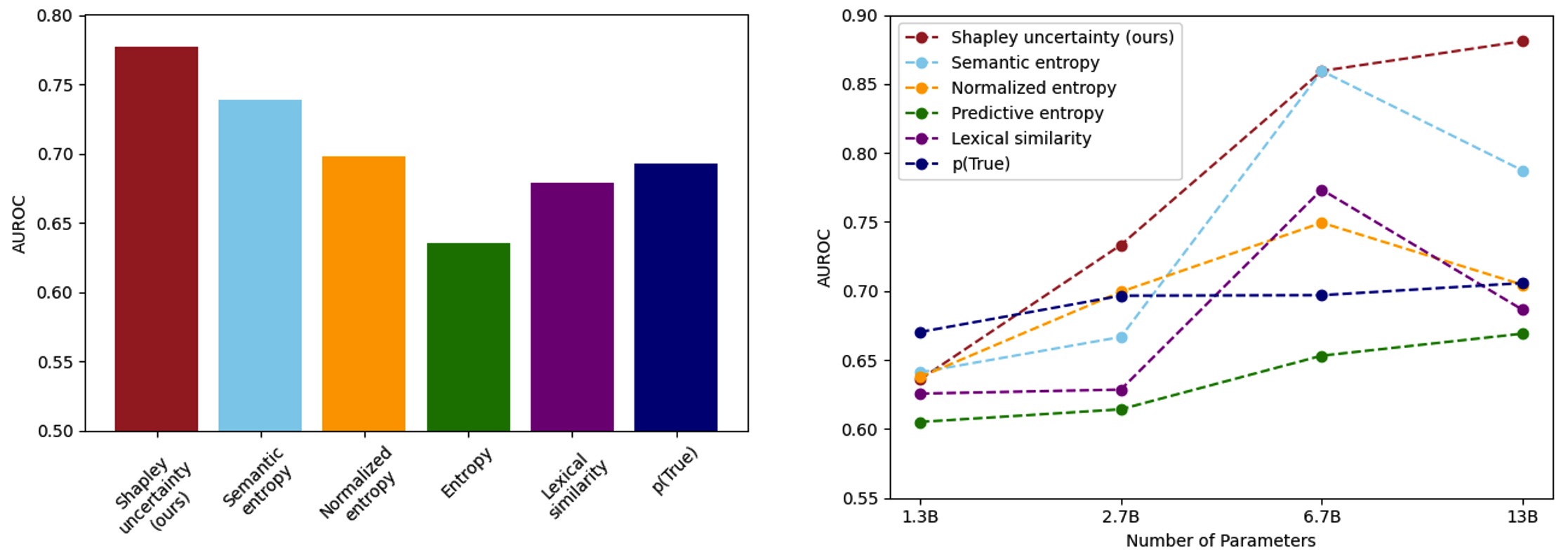}
\centering
\caption{
(\textbf{Left}) Our Shapley uncertainty (red) predicts model accuracy better than baselines on the answering dataset CoQA (average performance across OPT models from 1.3B to 13B parameters). 
(\textbf{Right}) The outstanding performance of Shapley uncertainty (red) becomes more pronounced with larger model sizes, while still maintaining effectiveness for smaller models.
}
\label{fig:2}
\vspace{-3mm}
\end{figure}

Regarding \textbf{models}, following \citet{kuhn2022semantic,liu2024uncertainty}, we utilize GPT-like OPT models \citep{zhang2022opt} of varying sizes, from 2.7B to 13B parameters. 
We also conduct experiments on Gemma-7B, LLaMA2-7B~\citep{touvron2023llama}, and LLaMA3-8B.
Falcon-7B~\citep{falcon40b}, Mistral-7B~\citep{jiang2023mistral7b}, Qwen2.5-14B~\citep{qwen2.5} and DeepSeek-R1-Distill-LLaMA-8B\citep{deepseekai2025deepseekr1incentivizingreasoningcapability} are also selected for systematically evaluating architectural variations across different design paradigms of SOTA large language models.

\label{baseline}
We evaluate our approach against several established \textbf{baselines}, namely predictive entropy \citep{kadavath2022language}, 
length-normalized predictive entropy \citep{malinin2020uncertainty}, 
lexical similarity measures \citep{fomicheva2020unsupervised}, and semantic entropy \citep{kuhn2022semantic}. 
Predictive entropy here is simply the entropy of the output distribution. 
While it has traditionally been employed in various domains as a standard gauge of uncertainty, such applications, including those by \citet{kadavath2022language}, have not incorporated length-normalization. 
Length-normalized predictive entropy, on the other hand, adjusts the joint log-probability of sequences by their length, a refinement suggested by \citet{malinin2020uncertainty} to account for NLG uncertainty. 
The p(True) metric, introduced by \citet{kadavath2022language}, seeks to quantify the likelihood of a model’s output being accurate by essentially querying the model on the veracity of its own answers, involving the generation of multiple answers and subsequent questioning for validation. 
Lexical similarity, meanwhile, is assessed by averaging the similarity of the answers within a set. 
Semantic entropy incorporates linguistic invariances created by common meanings.
Shapley uncertainty further considers the correlation between different generated sentences.


\begin{figure}[t!]
\includegraphics[width=0.48
\textwidth]{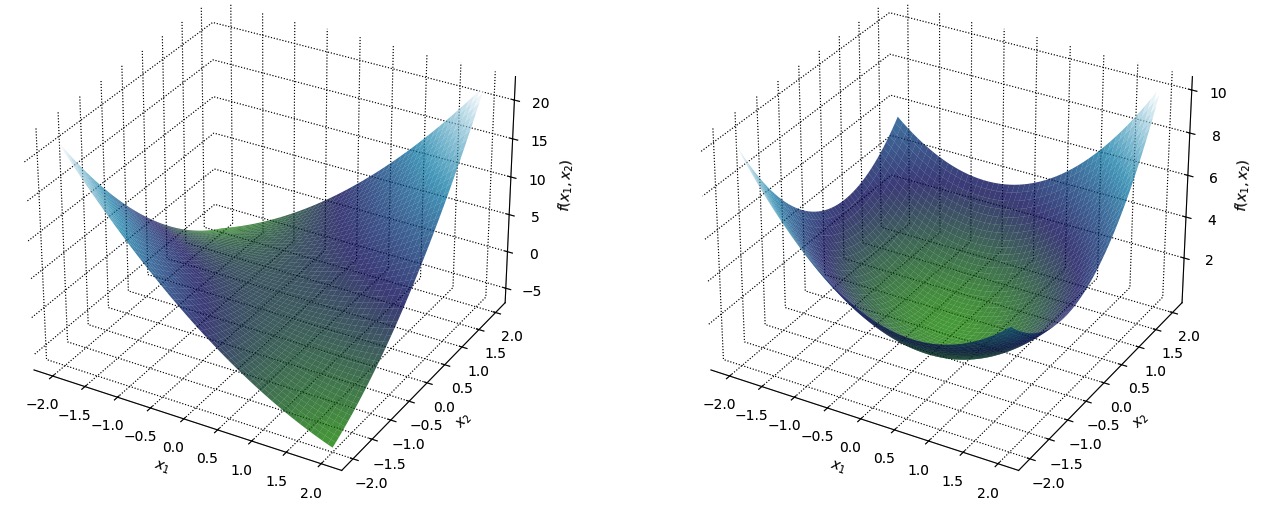}
\centering
\caption{
The figure illustrates the quadratic form for correlation matrices, 
demonstrating how our method transforms an inappropriate (non-positive semi-definite) correlation matrix (\textbf{Left}) into an appropriate (positive semi-definite) correlation matrix (\textbf{Right}).
}
\label{fig:3}
\vspace{-3mm}
\end{figure}

\begin{table*}[t!]
\centering
\label{tab:3}
\vspace{0mm}
\renewcommand\arraystretch{1.35}
\scalebox{0.8}{
\begin{tabular}{cccccccccccc}
\specialrule{.1em}{.075em}{.075em} 
\multirow{2}{*}{ Dataset } && \multirow{2}{*}{ LLM } && \multicolumn{6}{c}{ Benchmarks } && Ours \\
&& && MaxL & AvgL & MaxE & AvgE & SU & A4C  && Shap \\
\cline{1-1} \cline{3-3} \cline{5-10} \cline{12-12}
\multirow{3}{*}{ CoQA } && G-7B && 0.710 & 0.708 & 0.725 & 0.708 & 0.674 & 0.515 && \bf 0.739  \\
&& L-7B && 0.535 & 0.600 & 0.603 & 0.580 & 0.541 & 0.502 && \bf 0.712  \\
&& L-8B && 0.692 & 0.697 & 0.716 & 0.699 & 0.684 & 0.506 && \bf 0.756  \\
\cline{1-1} \cline{3-3} \cline{5-10} \cline{12-12}
\multirow{3}{*}{ WMT-14 } && G-7B && 0.668 & 0.589 & 0.637 & \bf  0.811 & 0.572 & 0.596 && 0.808  \\
&& L-7B && 0.606 & 0.712 & 0.583 & 0.711 & 0.513 & 0.506 && \bf  0.784  \\
&& L-8B && 0.554 & 0.685 & 0.616 & 0.729 & 0.510 & 0.502 && \bf 0.735  \\
\cline{1-1} \cline{3-3} \cline{5-10} \cline{12-12}
Average && - && 0.627 & 0.665 & 0.646 & 0.706 & 0.582 & 0.521 &&\bf  0.755 \\
\specialrule{.1em}{.075em}{.075em}
\end{tabular}
}
\caption{Out-of-sample AUROC performance for benchmarks and our methods on NLG tasks. G-7B, L-7B, and L-8B represent Gemma-7B, LLaMA2-7B, and LLaMA3-8B, respectively.
The methods MaxL, AvgL, MaxE, and AvgE are derived from the work of  \citet{manakul2023selfcheckgpt}. 
SU represents the semantic uncertainty estimation by \citet{kuhn2022semantic}, while the A4C column employs the ask-for-confidence approach developed by  \citet{tian2023just}. 
}
\label{tab:3}
\vspace{-1mm}
\end{table*}

For CoQA, Shapley uncertainty demonstrates superior performance over baseline methods in predicting the correctness of a model's answers. 
As shown in \Cref{fig:2} (right), using the larger-size models, Shapley uncertainty achieves a significantly higher AUROC compared to entropy (both with and without length-normalization), semantic entropy, and the lexical similarity baseline. 
Moreover, it substantially outperforms the p(True) metric. 
\Cref{fig:2} (left) shows that our method maintains its superiority in overall performance across models from 1.3B to 13B.

As shown in \Cref{fig:2.5}, our method demonstrates robust performance across diverse LLM architectures and datasets, validating its generalization capability for a wide range of question-answering tasks.
Comprehensive experimental results are detailed in \Cref{app:d}.

We conduct further experiments to compare Shapley uncertainty with the methods of MaxL, AvgL, MaxE, and AvgE which are proposed in \citet{manakul2023selfcheckgpt}, and the A4C method which employs the ask-for-confidence approach developed by \citet{tian2023just}.


\Cref{tab:3} presents a comparative analysis of our proposed Shapley uncertainty method against existing benchmarks, using AUROC as the evaluation metric. The numerical results yield several notable observations:
Firstly, Shapley uncertainty demonstrates superior performance in the majority of cases;
Secondly, there is one exception: for the Gemma-7B model on the WMT-14 dataset, the AvgE method achieves a marginally higher score (0.811) compared to Shapley uncertainty (0.808);
Thirdly, when considering average performance across all datasets and models, Shapley uncertainty exhibits a clear advantage over all other methods.
These findings underscore the robustness and effectiveness of the Shapley uncertainty method across diverse scenarios, with only minor exceptions.

In addition, as illustrated in \Cref{fig:3}, our method effectively transforms an inappropriate (non-positive semi-definite) correlation matrix into an appropriate (positive semi-definite) one. 
This transformation ensures that the resulting correlation matrix adheres to mathematical standards, a crucial aspect that may be overlooked by other methods.
More details are given in \Cref{app:b}.

\section{Limitation}

This work exhibits two primary limitations:
Firstly, the kernel function's ability to generate positive semi-definite correlation matrices is limited to finite sentence sets, reflecting a fundamental constraint in DNN-based models' capacity to produce mathematically well-defined correlations.
Secondly, computational constraints restricted our analysis to language models under 30 billion parameters, leaving the behavior of larger LLMs unexplored.
These limitations suggest promising directions for future research, particularly in developing more sophisticated correlation measures and extending the analysis to very large language models with enhanced computational resources.

\section{Conclusion}

We propose --- Shapley uncertainty --- a comprehensive approach to measure NLG outputs uncertainty, which extends beyond a simple threshold of shared meanings to encompass the intricate correlations among diverse sentences.
We employ a variant of the kernel method to ensure the correlation matrix satisfies Mercer's condition, a fundamental requirement for any valid correlation or distance metric. 
To address the challenges of uncertainty calculation in high-dimensional spaces, we apply the Shapley method to the resulting correlation matrix.
Our extensive empirical results demonstrate that the Shapley uncertainty metric more accurately predicts LLM performance on diverse datasets compared to existing baselines. 

\bibliography{custom}

\begin{thebibliography}{56}
\providecommand{\natexlab}[1]{#1}

\bibitem[{Abdar et~al.(2021)Abdar, Pourpanah, Hussain, Rezazadegan, Liu, Ghavamzadeh, Fieguth, Cao, Khosravi, Acharya et~al.}]{abdar2021review}
Moloud Abdar, Farhad Pourpanah, Sadiq Hussain, Dana Rezazadegan, Li~Liu, Mohammad Ghavamzadeh, Paul Fieguth, Xiaochun Cao, Abbas Khosravi, U~Rajendra Acharya, et~al. 2021.
\newblock A review of uncertainty quantification in deep learning: Techniques, applications and challenges.
\newblock \emph{Information fusion}, 76:243--297.

\bibitem[{Almazrouei et~al.(2023)Almazrouei, Alobeidli, Alshamsi, Cappelli, Cojocaru, Debbah, Goffinet, Heslow, Launay, Malartic, Noune, Pannier, and Penedo}]{falcon40b}
Ebtesam Almazrouei, Hamza Alobeidli, Abdulaziz Alshamsi, Alessandro Cappelli, Ruxandra Cojocaru, Merouane Debbah, Etienne Goffinet, Daniel Heslow, Julien Launay, Quentin Malartic, Badreddine Noune, Baptiste Pannier, and Guilherme Penedo. 2023.
\newblock {Falcon-40B}: an open large language model with state-of-the-art performance.

\bibitem[{Azaria and Mitchell(2023)}]{azaria2023internal}
Amos Azaria and Tom Mitchell. 2023.
\newblock The internal state of an llm knows when it's lying.
\newblock \emph{arXiv preprint arXiv:2304.13734}.

\bibitem[{Bojar et~al.(2014)Bojar, Buck, Federmann, Haddow, Koehn, Leveling, Monz, Pecina, Post, Saint-Amand et~al.}]{bojar2014findings}
Ond{\v{r}}ej Bojar, Christian Buck, Christian Federmann, Barry Haddow, Philipp Koehn, Johannes Leveling, Christof Monz, Pavel Pecina, Matt Post, Herve Saint-Amand, et~al. 2014.
\newblock Findings of the 2014 workshop on statistical machine translation.
\newblock In \emph{Proceedings of the ninth workshop on statistical machine translation}.

\bibitem[{Brown et~al.(2020)Brown, Mann, Ryder, Subbiah, Kaplan, Dhariwal, Neelakantan, Shyam, Sastry, Askell et~al.}]{brown2020language}
Tom Brown, Benjamin Mann, Nick Ryder, Melanie Subbiah, Jared~D Kaplan, Prafulla Dhariwal, Arvind Neelakantan, Pranav Shyam, Girish Sastry, Amanda Askell, et~al. 2020.
\newblock Language models are few-shot learners.
\newblock \emph{Advances in neural information processing systems}.

\bibitem[{CH-Wang et~al.(2023)CH-Wang, Van~Durme, Eisner, and Kedzie}]{ch2023androids}
Sky CH-Wang, Benjamin Van~Durme, Jason Eisner, and Chris Kedzie. 2023.
\newblock Do androids know they're only dreaming of electric sheep?
\newblock \emph{arXiv preprint arXiv:2312.17249}.

\bibitem[{Chen et~al.(2024)Chen, Liu, Chen, Gu, Wu, Tao, Fu, and Ye}]{chen2024inside}
Chao Chen, Kai Liu, Ze~Chen, Yi~Gu, Yue Wu, Mingyuan Tao, Zhihang Fu, and Jieping Ye. 2024.
\newblock Inside: Llms' internal states retain the power of hallucination detection.
\newblock \emph{arXiv preprint arXiv:2402.03744}.

\bibitem[{Chowdhery et~al.(2023)Chowdhery, Narang, Devlin, Bosma, Mishra, Roberts, Barham, Chung, Sutton, Gehrmann et~al.}]{chowdhery2023palm}
Aakanksha Chowdhery, Sharan Narang, Jacob Devlin, Maarten Bosma, Gaurav Mishra, Adam Roberts, Paul Barham, Hyung~Won Chung, Charles Sutton, Sebastian Gehrmann, et~al. 2023.
\newblock Palm: Scaling language modeling with pathways.
\newblock \emph{Journal of Machine Learning Research}.

\bibitem[{DeepSeek-AI et~al.(2025)DeepSeek-AI, Guo, Yang, Zhang, Song, Zhang, Xu, Zhu, Ma, Wang, Bi, Zhang, Yu, Wu, Wu, Gou, Shao, Li, Gao, Liu, Xue, Wang, Wu, Feng, Lu, Zhao, Deng, Zhang, Ruan, Dai, Chen, Ji, Li, Lin, Dai, Luo, Hao, Chen, Li, Zhang, Bao, Xu, Wang, Ding, Xin, Gao, Qu, Li, Guo, Li, Wang, Chen, Yuan, Qiu, Li, Cai, Ni, Liang, Chen, Dong, Hu, Gao, Guan, Huang, Yu, Wang, Zhang, Zhao, Wang, Zhang, Xu, Xia, Zhang, Zhang, Tang, Li, Wang, Li, Tian, Huang, Zhang, Wang, Chen, Du, Ge, Zhang, Pan, Wang, Chen, Jin, Chen, Lu, Zhou, Chen, Ye, Wang, Yu, Zhou, Pan, Li, Zhou, Wu, Ye, Yun, Pei, Sun, Wang, Zeng, Zhao, Liu, Liang, Gao, Yu, Zhang, Xiao, An, Liu, Wang, Chen, Nie, Cheng, Liu, Xie, Liu, Yang, Li, Su, Lin, Li, Jin, Shen, Chen, Sun, Wang, Song, Zhou, Wang, Shan, Li, Wang, Wei, Zhang, Xu, Li, Zhao, Sun, Wang, Yu, Zhang, Shi, Xiong, He, Piao, Wang, Tan, Ma, Liu, Guo, Ou, Wang, Gong, Zou, He, Xiong, Luo, You, Liu, Zhou, Zhu, Xu, Huang, Li, Zheng, Zhu, Ma, Tang, Zha, Yan, Ren, Ren, Sha, Fu, Xu, Xie, Zhang,
  Hao, Ma, Yan, Wu, Gu, Zhu, Liu, Li, Xie, Song, Pan, Huang, Xu, Zhang, and Zhang}]{deepseekai2025deepseekr1incentivizingreasoningcapability}
DeepSeek-AI, Daya Guo, Dejian Yang, Haowei Zhang, Junxiao Song, Ruoyu Zhang, Runxin Xu, Qihao Zhu, Shirong Ma, Peiyi Wang, Xiao Bi, Xiaokang Zhang, Xingkai Yu, Yu~Wu, Z.~F. Wu, Zhibin Gou, Zhihong Shao, Zhuoshu Li, Ziyi Gao, Aixin Liu, Bing Xue, Bingxuan Wang, Bochao Wu, Bei Feng, Chengda Lu, Chenggang Zhao, Chengqi Deng, Chenyu Zhang, Chong Ruan, Damai Dai, Deli Chen, Dongjie Ji, Erhang Li, Fangyun Lin, Fucong Dai, Fuli Luo, Guangbo Hao, Guanting Chen, Guowei Li, H.~Zhang, Han Bao, Hanwei Xu, Haocheng Wang, Honghui Ding, Huajian Xin, Huazuo Gao, Hui Qu, Hui Li, Jianzhong Guo, Jiashi Li, Jiawei Wang, Jingchang Chen, Jingyang Yuan, Junjie Qiu, Junlong Li, J.~L. Cai, Jiaqi Ni, Jian Liang, Jin Chen, Kai Dong, Kai Hu, Kaige Gao, Kang Guan, Kexin Huang, Kuai Yu, Lean Wang, Lecong Zhang, Liang Zhao, Litong Wang, Liyue Zhang, Lei Xu, Leyi Xia, Mingchuan Zhang, Minghua Zhang, Minghui Tang, Meng Li, Miaojun Wang, Mingming Li, Ning Tian, Panpan Huang, Peng Zhang, Qiancheng Wang, Qinyu Chen, Qiushi Du, Ruiqi Ge, Ruisong
  Zhang, Ruizhe Pan, Runji Wang, R.~J. Chen, R.~L. Jin, Ruyi Chen, Shanghao Lu, Shangyan Zhou, Shanhuang Chen, Shengfeng Ye, Shiyu Wang, Shuiping Yu, Shunfeng Zhou, Shuting Pan, S.~S. Li, Shuang Zhou, Shaoqing Wu, Shengfeng Ye, Tao Yun, Tian Pei, Tianyu Sun, T.~Wang, Wangding Zeng, Wanjia Zhao, Wen Liu, Wenfeng Liang, Wenjun Gao, Wenqin Yu, Wentao Zhang, W.~L. Xiao, Wei An, Xiaodong Liu, Xiaohan Wang, Xiaokang Chen, Xiaotao Nie, Xin Cheng, Xin Liu, Xin Xie, Xingchao Liu, Xinyu Yang, Xinyuan Li, Xuecheng Su, Xuheng Lin, X.~Q. Li, Xiangyue Jin, Xiaojin Shen, Xiaosha Chen, Xiaowen Sun, Xiaoxiang Wang, Xinnan Song, Xinyi Zhou, Xianzu Wang, Xinxia Shan, Y.~K. Li, Y.~Q. Wang, Y.~X. Wei, Yang Zhang, Yanhong Xu, Yao Li, Yao Zhao, Yaofeng Sun, Yaohui Wang, Yi~Yu, Yichao Zhang, Yifan Shi, Yiliang Xiong, Ying He, Yishi Piao, Yisong Wang, Yixuan Tan, Yiyang Ma, Yiyuan Liu, Yongqiang Guo, Yuan Ou, Yuduan Wang, Yue Gong, Yuheng Zou, Yujia He, Yunfan Xiong, Yuxiang Luo, Yuxiang You, Yuxuan Liu, Yuyang Zhou, Y.~X. Zhu,
  Yanhong Xu, Yanping Huang, Yaohui Li, Yi~Zheng, Yuchen Zhu, Yunxian Ma, Ying Tang, Yukun Zha, Yuting Yan, Z.~Z. Ren, Zehui Ren, Zhangli Sha, Zhe Fu, Zhean Xu, Zhenda Xie, Zhengyan Zhang, Zhewen Hao, Zhicheng Ma, Zhigang Yan, Zhiyu Wu, Zihui Gu, Zijia Zhu, Zijun Liu, Zilin Li, Ziwei Xie, Ziyang Song, Zizheng Pan, Zhen Huang, Zhipeng Xu, Zhongyu Zhang, and Zhen Zhang. 2025.
\newblock \href {https://arxiv.org/abs/2501.12948} {Deepseek-r1: Incentivizing reasoning capability in llms via reinforcement learning}.
\newblock \emph{Preprint}, arXiv:2501.12948.

\bibitem[{Desai and Durrett(2020)}]{desai2020calibration}
Shrey Desai and Greg Durrett. 2020.
\newblock Calibration of pre-trained transformers.
\newblock \emph{arXiv preprint arXiv:2003.07892}.

\bibitem[{Duan et~al.(2024)Duan, Yang, and Tam}]{duan2024llms}
Hanyu Duan, Yi~Yang, and Kar~Yan Tam. 2024.
\newblock Do llms know about hallucination? an empirical investigation of llm's hidden states.
\newblock \emph{arXiv preprint arXiv:2402.09733}.

\bibitem[{Duan et~al.(2023)Duan, Cheng, Wang, Wang, Zavalny, Xu, Kailkhura, and Xu}]{duan2023shifting}
Jinhao Duan, Hao Cheng, Shiqi Wang, Chenan Wang, Alex Zavalny, Renjing Xu, Bhavya Kailkhura, and Kaidi Xu. 2023.
\newblock Shifting attention to relevance: Towards the uncertainty estimation of large language models.
\newblock \emph{arXiv preprint arXiv:2307.01379}.

\bibitem[{Farquhar et~al.(2024)Farquhar, Kossen, Kuhn, and Gal}]{farquhar2024detecting}
Sebastian Farquhar, Jannik Kossen, Lorenz Kuhn, and Yarin Gal. 2024.
\newblock Detecting hallucinations in large language models using semantic entropy.
\newblock \emph{Nature}, 630(8017):625--630.

\bibitem[{Filos et~al.(2019)Filos, Farquhar, Gomez, Rudner, Kenton, Smith, Alizadeh, de~Kroon, and Gal}]{filos2019benchmarking}
Angelos Filos, Sebastian Farquhar, Aidan~N Gomez, Tim~GJ Rudner, Zachary Kenton, Lewis Smith, Milad Alizadeh, Arnoud de~Kroon, and Yarin Gal. 2019.
\newblock Benchmarking bayesian deep learning with diabetic retinopathy diagnosis.
\newblock \emph{Preprint at https://arxiv. org/abs/1912.10481}.

\bibitem[{Fomicheva et~al.(2020)Fomicheva, Sun, Yankovskaya, Blain, Guzm{\'a}n, Fishel, Aletras, Chaudhary, and Specia}]{fomicheva2020unsupervised}
Marina Fomicheva, Shuo Sun, Lisa Yankovskaya, Fr{\'e}d{\'e}ric Blain, Francisco Guzm{\'a}n, Mark Fishel, Nikolaos Aletras, Vishrav Chaudhary, and Lucia Specia. 2020.
\newblock Unsupervised quality estimation for neural machine translation.
\newblock \emph{Transactions of the Association for Computational Linguistics}, 8:539--555.

\bibitem[{Gawlikowski et~al.(2023)Gawlikowski, Tassi, Ali, Lee, Humt, Feng, Kruspe, Triebel, Jung, Roscher et~al.}]{gawlikowski2023survey}
Jakob Gawlikowski, Cedrique Rovile~Njieutcheu Tassi, Mohsin Ali, Jongseok Lee, Matthias Humt, Jianxiang Feng, Anna Kruspe, Rudolph Triebel, Peter Jung, Ribana Roscher, et~al. 2023.
\newblock A survey of uncertainty in deep neural networks.
\newblock \emph{Artificial Intelligence Review}, 56(Suppl 1):1513--1589.

\bibitem[{He et~al.(2020)He, Liu, Gao, and Chen}]{he2020deberta}
Pengcheng He, Xiaodong Liu, Jianfeng Gao, and Weizhu Chen. 2020.
\newblock Deberta: Decoding-enhanced bert with disentangled attention.
\newblock In \emph{International Conference on Learning Representations}.

\bibitem[{Hendrycks et~al.(2021)Hendrycks, Carlini, Schulman, and Steinhardt}]{hendrycks2021unsolved}
Dan Hendrycks, Nicholas Carlini, John Schulman, and Jacob Steinhardt. 2021.
\newblock Unsolved problems in ml safety.
\newblock \emph{arXiv preprint arXiv:2109.13916}.

\bibitem[{Hoffmann et~al.(2022)Hoffmann, Borgeaud, Mensch, Buchatskaya, Cai, Rutherford, Casas, Hendricks, Welbl, Clark et~al.}]{hoffmann2022training}
Jordan Hoffmann, Sebastian Borgeaud, Arthur Mensch, Elena Buchatskaya, Trevor Cai, Eliza Rutherford, Diego de~Las Casas, Lisa~Anne Hendricks, Johannes Welbl, Aidan Clark, et~al. 2022.
\newblock Training compute-optimal large language models.
\newblock \emph{arXiv preprint arXiv:2203.15556}.

\bibitem[{Jiang et~al.(2023)Jiang, Sablayrolles, Mensch, Bamford, Chaplot, de~las Casas, Bressand, Lengyel, Lample, Saulnier, Lavaud, Lachaux, Stock, Scao, Lavril, Wang, Lacroix, and Sayed}]{jiang2023mistral7b}
Albert~Q. Jiang, Alexandre Sablayrolles, Arthur Mensch, Chris Bamford, Devendra~Singh Chaplot, Diego de~las Casas, Florian Bressand, Gianna Lengyel, Guillaume Lample, Lucile Saulnier, Lélio~Renard Lavaud, Marie-Anne Lachaux, Pierre Stock, Teven~Le Scao, Thibaut Lavril, Thomas Wang, Timothée Lacroix, and William~El Sayed. 2023.
\newblock \href {https://arxiv.org/abs/2310.06825} {Mistral 7b}.
\newblock \emph{Preprint}, arXiv:2310.06825.

\bibitem[{Jin et~al.(2025)Jin, Yi, Huang, Schewe, and Huang}]{jin2025s}
Gaojie Jin, Xinping Yi, Wei Huang, Sven Schewe, and Xiaowei Huang. 2025.
\newblock S-o: Enhancing adversarial training with second-order statistics of weights.
\newblock \emph{IEEE Transactions on Pattern Analysis \& Machine Intelligence}, (01):1--15.

\bibitem[{Jin et~al.(2020)Jin, Yi, Zhang, Zhang, Schewe, and Huang}]{jin2020does}
Gaojie Jin, Xinping Yi, Liang Zhang, Lijun Zhang, Sven Schewe, and Xiaowei Huang. 2020.
\newblock How does weight correlation affect generalisation ability of deep neural networks?
\newblock \emph{Advances in Neural Information Processing Systems}, 33:21346--21356.

\bibitem[{Joshi et~al.(2017)Joshi, Choi, Weld, and Zettlemoyer}]{joshi2017triviaqa}
Mandar Joshi, Eunsol Choi, Daniel~S Weld, and Luke Zettlemoyer. 2017.
\newblock Triviaqa: A large scale distantly supervised challenge dataset for reading comprehension.
\newblock \emph{arXiv preprint arXiv:1705.03551}.

\bibitem[{Kadavath et~al.(2022)Kadavath, Conerly, Askell, Henighan, Drain, Perez, Schiefer, Hatfield-Dodds, DasSarma, Tran-Johnson et~al.}]{kadavath2022language}
Saurav Kadavath, Tom Conerly, Amanda Askell, Tom Henighan, Dawn Drain, Ethan Perez, Nicholas Schiefer, Zac Hatfield-Dodds, Nova DasSarma, Eli Tran-Johnson, et~al. 2022.
\newblock Language models (mostly) know what they know.
\newblock \emph{arXiv preprint arXiv:2207.05221}.

\bibitem[{Kendall and Gal(2017)}]{kendall2017uncertainties}
Alex Kendall and Yarin Gal. 2017.
\newblock What uncertainties do we need in bayesian deep learning for computer vision?
\newblock \emph{Advances in neural information processing systems}, 30.

\bibitem[{Krithara et~al.(2023)Krithara, Nentidis, Bougiatiotis, and Paliouras}]{bioasq}
Anastasia Krithara, Anastasios Nentidis, Konstantinos Bougiatiotis, and Georgios Paliouras. 2023.
\newblock \href {https://doi.org/10.1038/s41597-023-02068-4} {Bioasq-qa: A manually curated corpus for biomedical question answering}.
\newblock \emph{Scientific Data}, 10:170.

\bibitem[{Kuhn et~al.(2023)Kuhn, Gal, and Farquhar}]{kuhn2022semantic}
Lorenz Kuhn, Yarin Gal, and Sebastian Farquhar. 2023.
\newblock Semantic uncertainty: Linguistic invariances for uncertainty estimation in natural language generation.
\newblock In \emph{International Conference on Learning Representations}.

\bibitem[{Kumar et~al.(2023)Kumar, Lu, Gupta, Palepu, Bellamy, Raskar, and Beam}]{kumar2023conformal}
Bhawesh Kumar, Charlie Lu, Gauri Gupta, Anil Palepu, David Bellamy, Ramesh Raskar, and Andrew Beam. 2023.
\newblock Conformal prediction with large language models for multi-choice question answering.
\newblock \emph{arXiv preprint arXiv:2305.18404}.

\bibitem[{Lahlou et~al.(2022)Lahlou, Jain, Nekoei, Butoi, Bertin, Rector-Brooks, Korablyov, and Bengio}]{lahlou2022deup}
Salem Lahlou, Moksh Jain, Hadi Nekoei, Victor~I Butoi, Paul Bertin, Jarrid Rector-Brooks, Maksym Korablyov, and Yoshua Bengio. 2022.
\newblock Deup: Direct epistemic uncertainty prediction.
\newblock \emph{Transactions on Machine Learning Research}.

\bibitem[{Lee et~al.(2019)Lee, Chang, and Toutanova}]{lee-etal-2019-latent}
Kenton Lee, Ming-Wei Chang, and Kristina Toutanova. 2019.
\newblock \href {https://doi.org/10.18653/v1/P19-1612} {Latent retrieval for weakly supervised open domain question answering}.
\newblock In \emph{Proceedings of the 57th Annual Meeting of the Association for Computational Linguistics}, pages 6086--6096, Florence, Italy. Association for Computational Linguistics.

\bibitem[{Lin and Och(2004{\natexlab{a}})}]{lin2004automatic}
Chin-Yew Lin and Franz~Josef Och. 2004{\natexlab{a}}.
\newblock Automatic evaluation of machine translation quality using longest common subsequence and skip-bigram statistics.
\newblock In \emph{Proceedings of the 42nd annual meeting of the association for computational linguistics (ACL-04)}, pages 605--612.

\bibitem[{Lin and Och(2004{\natexlab{b}})}]{lin2004orange}
Chin-Yew Lin and Franz~Josef Och. 2004{\natexlab{b}}.
\newblock Orange: a method for evaluating automatic evaluation metrics for machine translation.
\newblock In \emph{COLING 2004: Proceedings of the 20th International Conference on Computational Linguistics}, pages 501--507.

\bibitem[{Lin et~al.(2022{\natexlab{a}})Lin, Hilton, and Evans}]{lin2022teaching}
Stephanie Lin, Jacob Hilton, and Owain Evans. 2022{\natexlab{a}}.
\newblock Teaching models to express their uncertainty in words.
\newblock \emph{Transactions on Machine Learning Research}.

\bibitem[{Lin et~al.(2022{\natexlab{b}})Lin, Liu, and Shang}]{lin2022towards}
Zi~Lin, Jeremiah~Zhe Liu, and Jingbo Shang. 2022{\natexlab{b}}.
\newblock Towards collaborative neural-symbolic graph semantic parsing via uncertainty.
\newblock \emph{Findings of the Association for Computational Linguistics: ACL 2022}.

\bibitem[{Liu et~al.(2024)Liu, Pan, Li, and Chen}]{liu2024uncertainty}
Linyu Liu, Yu~Pan, Xiaocheng Li, and Guanting Chen. 2024.
\newblock Uncertainty estimation and quantification for llms: A simple supervised approach.
\newblock \emph{arXiv preprint arXiv:2404.15993}.

\bibitem[{Malinin and Gales(2020)}]{malinin2020uncertainty}
Andrey Malinin and Mark Gales. 2020.
\newblock Uncertainty estimation in autoregressive structured prediction.
\newblock In \emph{International Conference on Learning Representations}.

\bibitem[{Manakul et~al.(2023)Manakul, Liusie, and Gales}]{manakul2023selfcheckgpt}
Potsawee Manakul, Adian Liusie, and Mark~JF Gales. 2023.
\newblock Selfcheckgpt: Zero-resource black-box hallucination detection for generative large language models.
\newblock \emph{arXiv preprint arXiv:2303.08896}.

\bibitem[{Mohri and Hashimoto(2024)}]{mohri2024language}
Christopher Mohri and Tatsunori Hashimoto. 2024.
\newblock Language models with conformal factuality guarantees.
\newblock \emph{arXiv preprint arXiv:2402.10978}.

\bibitem[{Papineni et~al.(2002)Papineni, Roukos, Ward, and Zhu}]{papineni2002bleu}
Kishore Papineni, Salim Roukos, Todd Ward, and Wei-Jing Zhu. 2002.
\newblock Bleu: a method for automatic evaluation of machine translation.
\newblock In \emph{Proceedings of the 40th annual meeting of the Association for Computational Linguistics}, pages 311--318.

\bibitem[{Patel et~al.(2021)Patel, Bhattamishra, and Goyal}]{patel-etal-2021-nlp}
Arkil Patel, Satwik Bhattamishra, and Navin Goyal. 2021.
\newblock \href {https://doi.org/10.18653/v1/2021.naacl-main.168} {Are {NLP} models really able to solve simple math word problems?}
\newblock In \emph{Proceedings of the 2021 Conference of the North American Chapter of the Association for Computational Linguistics: Human Language Technologies}, pages 2080--2094, Online. Association for Computational Linguistics.

\bibitem[{Plaut et~al.(2024)Plaut, Nguyen, and Trinh}]{plaut2024softmax}
Benjamin Plaut, Khanh Nguyen, and Tu~Trinh. 2024.
\newblock Softmax probabilities (mostly) predict large language model correctness on multiple-choice q\&a.
\newblock \emph{arXiv preprint arXiv:2402.13213}.

\bibitem[{Quach et~al.(2023)Quach, Fisch, Schuster, Yala, Sohn, Jaakkola, and Barzilay}]{quach2023conformal}
Victor Quach, Adam Fisch, Tal Schuster, Adam Yala, Jae~Ho Sohn, Tommi~S Jaakkola, and Regina Barzilay. 2023.
\newblock Conformal language modeling.
\newblock \emph{arXiv preprint arXiv:2306.10193}.

\bibitem[{Rajpurkar et~al.(2016)Rajpurkar, Zhang, Lopyrev, and Liang}]{rajpurkar-etal-2016-squad}
Pranav Rajpurkar, Jian Zhang, Konstantin Lopyrev, and Percy Liang. 2016.
\newblock \href {https://doi.org/10.18653/v1/D16-1264} {{SQ}u{AD}: 100,000+ questions for machine comprehension of text}.
\newblock In \emph{Proceedings of the 2016 Conference on Empirical Methods in Natural Language Processing}, pages 2383--2392, Austin, Texas. Association for Computational Linguistics.

\bibitem[{Reddy et~al.(2019)Reddy, Chen, and Manning}]{reddy2019coqa}
Siva Reddy, Danqi Chen, and Christopher~D Manning. 2019.
\newblock Coqa: A conversational question answering challenge.
\newblock \emph{Transactions of the Association for Computational Linguistics}, 7:249--266.

\bibitem[{Si et~al.(2022)Si, Zhao, Min, and Boyd-Graber}]{si2022re}
Chenglei Si, Chen Zhao, Sewon Min, and Jordan Boyd-Graber. 2022.
\newblock Re-examining calibration: The case of question answering.
\newblock \emph{arXiv preprint arXiv:2205.12507}.

\bibitem[{Slobodkin et~al.(2023)Slobodkin, Goldman, Caciularu, Dagan, and Ravfogel}]{slobodkin2023curious}
Aviv Slobodkin, Omer Goldman, Avi Caciularu, Ido Dagan, and Shauli Ravfogel. 2023.
\newblock The curious case of hallucinatory (un) answerability: Finding truths in the hidden states of over-confident large language models.
\newblock In \emph{Proceedings of the 2023 Conference on Empirical Methods in Natural Language Processing}, pages 3607--3625.

\bibitem[{Su et~al.(2024)Su, Wang, Ai, Hu, Wu, Zhou, and Liu}]{su2024unsupervised}
Weihang Su, Changyue Wang, Qingyao Ai, Yiran Hu, Zhijing Wu, Yujia Zhou, and Yiqun Liu. 2024.
\newblock Unsupervised real-time hallucination detection based on the internal states of large language models.
\newblock \emph{arXiv preprint arXiv:2403.06448}.

\bibitem[{Tian et~al.(2023)Tian, Mitchell, Zhou, Sharma, Rafailov, Yao, Finn, and Manning}]{tian2023just}
Katherine Tian, Eric Mitchell, Allan Zhou, Archit Sharma, Rafael Rafailov, Huaxiu Yao, Chelsea Finn, and Christopher~D Manning. 2023.
\newblock Just ask for calibration: Strategies for eliciting calibrated confidence scores from language models fine-tuned with human feedback.
\newblock \emph{arXiv preprint arXiv:2305.14975}.

\bibitem[{Touvron et~al.(2023)Touvron, Martin, Stone, Albert, Almahairi, Babaei, Bashlykov, Batra, Bhargava, Bhosale et~al.}]{touvron2023llama}
Hugo Touvron, Louis Martin, Kevin Stone, Peter Albert, Amjad Almahairi, Yasmine Babaei, Nikolay Bashlykov, Soumya Batra, Prajjwal Bhargava, Shruti Bhosale, et~al. 2023.
\newblock Llama 2: Open foundation and fine-tuned chat models.
\newblock \emph{arXiv preprint arXiv:2307.09288}.

\bibitem[{Verma et~al.(2023)Verma, Tran, Ali, and Min}]{verma2023reducing}
Shreyas Verma, Kien Tran, Yusuf Ali, and Guangyu Min. 2023.
\newblock Reducing llm hallucinations using epistemic neural networks.
\newblock \emph{arXiv preprint arXiv:2312.15576}.

\bibitem[{Xu et~al.(2024)Xu, Jain, and Kankanhalli}]{xu2024hallucination}
Ziwei Xu, Sanjay Jain, and Mohan Kankanhalli. 2024.
\newblock Hallucination is inevitable: An innate limitation of large language models.
\newblock \emph{arXiv preprint arXiv:2401.11817}.

\bibitem[{Yang et~al.(2025)Yang, Yu, Li, Liu, Huang, Huang, Jiang, Tu, Zhang, Zhou, Lin, Dang, Yang, Yu, Li, Sun, Zhu, Men, He, Xu, Yin, Yu, Qiu, Ren, Yang, Li, Xu, and Zhang}]{qwen2.5}
An~Yang, Bowen Yu, Chengyuan Li, Dayiheng Liu, Fei Huang, Haoyan Huang, Jiandong Jiang, Jianhong Tu, Jianwei Zhang, Jingren Zhou, Junyang Lin, Kai Dang, Kexin Yang, Le~Yu, Mei Li, Minmin Sun, Qin Zhu, Rui Men, Tao He, Weijia Xu, Wenbiao Yin, Wenyuan Yu, Xiafei Qiu, Xingzhang Ren, Xinlong Yang, Yong Li, Zhiying Xu, and Zipeng Zhang. 2025.
\newblock Qwen2.5-1m technical report.
\newblock \emph{arXiv preprint arXiv:2501.15383}.

\bibitem[{Ye and Durrett(2021)}]{ye2021can}
Xi~Ye and Greg Durrett. 2021.
\newblock Can explanations be useful for calibrating black box models?
\newblock \emph{arXiv preprint arXiv:2110.07586}.

\bibitem[{Zhang et~al.(2023)Zhang, Xu, Yang, Jin, Huang, and Zhang}]{zhang2023trajpac}
Liang Zhang, Nathaniel Xu, Pengfei Yang, Gaojie Jin, Cheng-Chao Huang, and Lijun Zhang. 2023.
\newblock Trajpac: Towards robustness verification of pedestrian trajectory prediction models.
\newblock In \emph{Proceedings of the IEEE/CVF International Conference on Computer Vision}, pages 8327--8339.

\bibitem[{Zhang et~al.(2021)Zhang, Gong, and Choi}]{zhang2021knowing}
Shujian Zhang, Chengyue Gong, and Eunsol Choi. 2021.
\newblock Knowing more about questions can help: Improving calibration in question answering.
\newblock \emph{arXiv preprint arXiv:2106.01494}.

\bibitem[{Zhang et~al.(2022)Zhang, Roller, Goyal, Artetxe, Chen, Chen, Dewan, Diab, Li, Lin et~al.}]{zhang2022opt}
Susan Zhang, Stephen Roller, Naman Goyal, Mikel Artetxe, Moya Chen, Shuohui Chen, Christopher Dewan, Mona Diab, Xian Li, Xi~Victoria Lin, et~al. 2022.
\newblock Opt: Open pre-trained transformer language models.
\newblock \emph{arXiv preprint arXiv:2205.01068}.

\end{thebibliography}

\clearpage
\appendix

\section{Proofs}
\label{app:a}
\begin{proof}[Proof for Proposition 3.1]
Assume $\R\in\mathbb{R}^{n\times n}$ is the correlation matrix generated by $\K(\D(\cdot,\cdot))$ with $\beta=1$, and $\R'\in\mathbb{R}^{n\times n}$ is the correlation matrix generated with some $\beta$, we need to to prove that there exists a parameter $\beta$ such that the matrix $\R'$ is positive semi-definite.

Since $\R$ is symmetric, all its eigenvalues are real.
Denote the eigenvalues of $\R$ by $\lambda_1,...,\lambda_n$.
We know $\R'=\mathbf{I}+\beta(\R-\bf I)$, where $\bf I$ is the identity matrix. 
Consider the matrix $\R'$ and analyze its eigenvalues,
if $\lambda_i$ is an eigenvalue of $\R$, then the corresponding eigenvalue of $\R'$  is given by $\lambda'_i=1+\beta(\lambda_i-1)$.

As a matrix is positive semi-definite if all its eigenvalues are non-negative.
To ensure $\lambda'_i\ge 0$, we have 
$$
1+\beta(\lambda_i-1)\ge 0.
$$
According to Gershgorin circle theorem, $|\lambda_i| \le n$.
Thus for all $0<\beta\le \frac{1}{n+1}$, we can ensure $\lambda'_i\ge 0$ and the matrix $\R'$ is positive semi-definite.
\end{proof}
\noindent
\emph{Remark. The hyper-parameter $\beta\in(0,\frac{1}{n+1}]$ ensures positive semi-definiteness of the correlation matrix even in the worst case. 
In practice, we set $\beta=0.5$ as our experiments demonstrate this value maintains positive semi-definiteness across all correlation matrices while avoiding the conservative bounds required for worst-case scenarios.
}

\begin{proof}[Proof for Proposition 3.6 (1)]
Let $\tilde {\bf s}=[{\bf s}_1,...,{\bf s}_n]$, $\phi(\cdot)$ be the Sharpley uncertainty metric defined in \eqref{eq:marginal uncertainty} and \eqref{eq:shapley uncertainty}.
For all ${\bf{s}}_i \in \tilde {\bf s}$ such that $$|\rho({\bf s}_i, \underset{{\bf s}_j \in \tilde {\bf s}}{\arg\min} \phi({\bf s}_j | \tilde {\bf s}))|<1,$$ 
suppose there exists a random variable $\g$ satisfies $$\rho(\g,\underset{{\bf s}_j \in \tilde {\bf s}}{\arg\min} \phi({\bf s}_j | \tilde {\bf s}))=1.$$
Then we have
$$
\phi({\bf s}_i | \tilde {\bf s}) > \phi(\g|\tilde {\bf s}\backslash {\bf s}_i\cup \g), 
$$
and for all ${\bf s}_j \in \tilde {\bf s}\backslash{\bf s}_i$, we get
$$
\phi({\bf s}_j | \tilde {\bf s}) \ge \phi({\bf s}_j|\tilde {\bf s}\backslash {\bf s}_i\cup \g). 
$$
Thus we have 
$$
\phi(\tilde {\bf s}) > \phi(\tilde {\bf s}\backslash {\bf s}_i\cup \g). 
$$
Shapley uncertainty metric satisfies Property~\ref{property1}.
\end{proof}

\begin{proof}[Proof for Proposition 3.6 (2)]
Let $\g$ be an independent random variable, then for all ${\bf{s}}_i \in \tilde {\bf s}$ such that 
$$\exists {\bf s}_j \in \tilde {\bf s}\backslash{\bf s}_i, \quad |\rho({\bf s}_i,{\bf s}_j)|>0,$$ 
we get
$$
\phi(\g|\tilde {\bf s}\backslash{\bf s}_i \cup \g) > \phi({\bf s}_i|\tilde {\bf s}),
$$
and for all ${\bf s}_j \in \tilde {\bf s}\backslash{\bf s}_i$, we have
$$
\phi({\bf s}_j | \tilde {\bf s}\backslash{\bf s}_i\cup \g) \ge \phi({\bf s}_j | \tilde {\bf s}).
$$
Thus, we get
$$
\phi(\tilde {\bf s}\backslash{\bf s}_i\cup \g) > \phi(\tilde {\bf s}). 
$$
Shapley uncertainty metric satisfies Property~\ref{property2}.
\end{proof}

\begin{proof}[Proof for Proposition 3.6 (3)]
Let $\g$ be a variable that may be correlated with elements in $\tilde {\bf s}$. 
For some ${\bf s}_j \in \tilde {\bf s}$, if the following inequality holds for all non-empty subsets $\X\subseteq\{1,...,n\}\backslash {j}$ with corresponding sequence ${\bf x}=[{\bf s}_i]_{i\in \X}$:
$$
\phi(\g|{\bf x}\cup {\g})>\phi({\bf s}_j|{\bf x}\cup{{\bf s}_j}),
$$
we get
$$
\phi(\g|\tilde {\bf s}\backslash{\bf s}_j\cup \g) > \phi({\bf s}_j|\tilde {\bf s}),
$$
and for all ${\bf s}_i\in \tilde {\bf s}\backslash{\bf s}_j$, we have
$$
\phi({\bf s}_i|\tilde {\bf s}\backslash{\bf s}_j\cup \g) \ge \phi({\bf s}_i|\tilde {\bf s}).
$$
Thus, 
$$
\phi(\tilde {\bf s}\backslash{\bf s}_j\cup \g) > \phi(\tilde {\bf s}).
$$
Shapley uncertainty metric satisfies Property~\ref{property3}.
\end{proof}

\section{More details of experiments}
\label{app:b}

We ran all of our experiments on 3 $\times$ NVIDIA A100 and 4 $\times$ RTX 4090 GPUs.

In the following, we provide detailed information on the two tasks utilized in our experiments.
\begin{itemize}
    \item Question answering. We follow \citet{kuhn2022semantic} and use the CoQA and TriviaQA datasets.
    We utilize a subset of 8000 questions from each, aligning with the size of CoQA's training split. Our evaluation metric of choice is a fuzzy matching criterion: $RougeL(\s,\s')>0.3$, implying that an answer is considered correct if its Rouge-L \citep{lin2004automatic} similarity to the reference answer exceeds 0.3.
    \item Machine translation. We utilize the WMT 2014 dataset \citep{bojar2014findings}, 
    and employ the BLEU score \citep{papineni2002bleu,lin2004orange} as our scoring function $BLEU(\cdot, \cdot)$. 
    Generated translations $\s$ are classified as correct if $BLEU(\s, \s_{\text{true}})$ exceeds 0.3, and incorrect otherwise.
\end{itemize}

\begin{figure*}[t!]
\includegraphics[width=0.68
\textwidth]{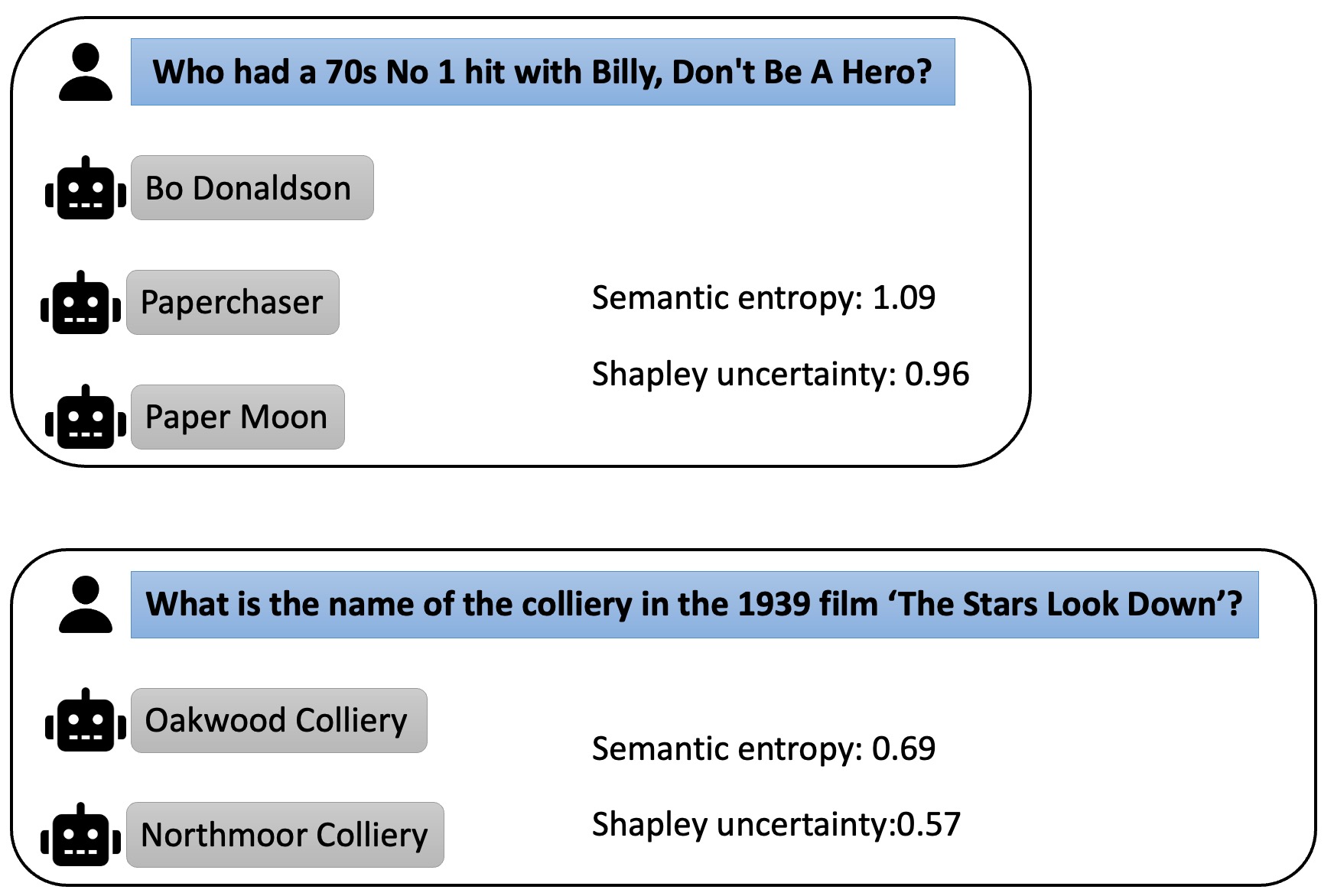}
\centering
\caption{
Two examples of Shapley uncertainty over the answers of LLaMA2-7B (\textbf{Upper}) and Gemma-7B (\textbf{Bottom}).
}
\label{fig:4}
\vspace{-3mm}
\end{figure*}

We also outline the various prompting templates employed across different tasks. 
We utilize few-shot prompting, with templates consistently structured into four key components:
\begin{itemize}
    \item Introduction: Provides context for the task (omitted only for WMT).
    \item Examples: Consists of $r$ distinct question-answer pairs, formatted identically to the subsequent question and answer sections.
    \item Question: Presents the specific query for the model to address.
    \item Answer: Reserved for the model's response.
\end{itemize}
The model receives as input the entire template string, excluding the reference answer. This approach ensures a consistent format across tasks while allowing for task-specific customization.

\quad

\emph{CoQA}

\emph{Reading the passage and answer given questions accordingly.}

\emph{Passage: a passage in COQA}

\emph{Examples:}

\emph{$r$ distinct question-answer pairs related to the given passage}

\emph{Q: a new question related to the given passage }

\emph{A: reference answer}

\quad

\emph{TriviaQA}

\emph{Answer the question as following examples.}

\emph{Examples:}

\emph{$r$ distinct question-answer pairs }

\emph{Q: a new question}

\emph{A: reference answer}

\quad

\emph{WMT}

\emph{$r$ distinct question-answer pairs}

\emph{Q: What is the English translation of the following sentence? (a French sentence)}

\emph{A: reference answer (an English sentence)}

\section{Related work}

\subsection{Uncertainty metric}

Uncertainty estimation for NLG has become a critical area of research. 
While uncertainty estimation and calibration in traditional machine learning are well-established \citep{abdar2021review,gawlikowski2023survey}, the rapid advancement of LLMs presents new challenges. 
There is an urgent need to better understand and measure uncertainty in LLMs' responses, particularly when dealing with variable-length sentence outputs rather than fixed-dimension data.
A significant body of work has emerged focusing on unsupervised methods for quantifying uncertainty. These approaches leverage various techniques: Entropy-based methods \citep{malinin2020uncertainty}; Similarity measures \citep{fomicheva2020unsupervised,lin2022towards}; Semantic analysis \citep{kuhn2022semantic,duan2023shifting,farquhar2024detecting}; Information from logits or hidden states \citep{kadavath2022language,chen2024inside,su2024unsupervised,plaut2024softmax}.
These techniques aim to craft effective uncertainty metrics. For black-box models, some metrics can be computed using multiple sampled outputs from LLMs. White-box models offer additional information, such as output distributions, logit values, and hidden layer data, facilitating easier computation of uncertainty metrics.
Other related uncertainty estimation methods, including calibration and conformal prediction, are discussed in works by \citet{desai2020calibration,zhang2021knowing,ye2021can,si2022re,quach2023conformal,kumar2023conformal,mohri2024language,jin2020does,jin2025s,zhang2023trajpac}.

This study focuses on developing an appropriate uncertainty metric for black-box models. 
Our approach distinguishes itself from previous works by addressing three critical aspects: (1) the correlation between sentences, (2) the positive semi-definiteness of the correlation matrix, and (3) the computation of uncertainty metrics in high-dimensional spaces. 
These considerations collectively represent a more comprehensive approach to uncertainty estimation, as prior studies have not simultaneously addressed all these factors.

\subsection{Uncertainty and hallucination detection}
A recent trend in the field of natural language processing involves leveraging uncertainty estimation techniques for hallucination detection in LLMs. 
This approach is grounded in the hypothesis that the values of logits and hidden states contain implicit information about the model's confidence in its generated output.
Several studies have explored this connection:
\begin{itemize}
    \item \citet{azaria2023internal} trained a classifier using hidden layer activations to predict hallucinations.
    \item \citet{verma2023reducing} developed epistemic neural networks aimed at mitigating hallucinations.
    \item \citet{slobodkin2023curious} demonstrated that hidden layer information from LLM outputs can indicate query answerability, indirectly informing hallucination detection.
    \item \citet{chen2024inside} created an unsupervised metric utilizing LLMs' internal states for hallucination detection.
\end{itemize}
Additional research in this area includes works by \citet{ch2023androids,duan2024llms,xu2024hallucination}.
While a universally accepted definition of hallucination remains elusive, with variations across the literature, uncertainty estimation offers a well-defined framework. 
Consequently, advancements in uncertainty estimation has the potential to significantly contribute to the task of hallucination detection.

\section{Supplementary results}
\label{app:d}

\begin{table*}[t!]
\centering
\vspace{0mm}
\renewcommand\arraystretch{1.35}
\scalebox{0.8}{
\begin{tabular}{ccccccccccc}
\specialrule{.1em}{.075em}{.075em} 
\multirow{2}{*}{ Dataset } && \multirow{2}{*}{ LLM } && \multicolumn{5}{c}{ Uncertainty Algorithms } && Ours \\
&& && SE & DSE & NE & P(T) & ER && Shap \\
\cline{1-1} \cline{3-3} \cline{5-9} \cline{11-11}
\multirow{6}{*}{ SQuAD } && L-7B && 0.748 & 0.741 & 0.705 & 0.602 & 0.618 && 0.733  \\
&& L-13B && 0.757 & 0.765 & 0.739 & 0.665 & 0.594 && 0.764  \\
&& F-7B && 0.691 & 0.679 & 0.692 & 0.492 & 0.634 && 0.670  \\
&& M-7B && 0.725 & 0.728 & 0.729 & 0.662 & 0.608 && 0.728  \\
&& Q-14B && 0.752 & 0.754 & 0.730 & 0.685 & 0.604 && 0.764  \\
&& D-8B && 0.792 & 0.797 & 0.744 & 0.667 & 0.656 && 0.787  \\
\cline{1-1} \cline{3-3} \cline{5-9} \cline{11-11}
\multirow{6}{*}{ BioASQ } && L-7B && 0.856 & 0.857 & 0.655 & 0.802 & 0.728 && 0.861  \\
&& L-13B && 0.838 & 0.828 & 0.723 & 0.789 & 0.752 && 0.819  \\
&& F-7B && 0.869 & 0.864 & 0.696 & 0.575 & 0.838 && 0.859  \\
&& M-7B && 0.894 & 0.889 & 0.764 & 0.738 & 0.811 && 0.892  \\
&& Q-14B && 0.855 & 0.845 & 0.803 & 0.694 & 0.769 && 0.867  \\
&& D-8B && 0.837 & 0.833 & 0.832 & 0.816 & 0.820 && 0.839  \\
\cline{1-1} \cline{3-3} \cline{5-9} \cline{11-11}
\multirow{6}{*}{ NQ-Open } && L-7B && 0.721 & 0.724 & 0.716 & 0.649 & 0.589 && 0.729  \\
&& L-13B && 0.715 & 0.720 & 0.686 & 0.685 & 0.625 && 0.714  \\
&& F-7B && 0.771 & 0.787 & 0.762 & 0.559 & 0.660 && 0.776  \\
&& M-7B && 0.765 & 0.767 & 0.728 & 0.714 & 0.565 && 0.778  \\
&& Q-14B && 0.815 & 0.820 & 0.775 & 0.714 & 0.658 && 0.812  \\
&& D-8B && 0.753 & 0.763 & 0.759 & 0.675 & 0.580 && 0.764  \\
\cline{1-1} \cline{3-3} \cline{5-9} \cline{11-11}
\multirow{6}{*}{ SVAMP } && L-7B && 0.825 & 0.826 & 0.813 & 0.594 & 0.786 && 0.830  \\
&& L-13B && 0.891 & 0.890 & 0.873 & 0.826 & 0.830 && 0.887 \\
&& F-7B && 0.734 & 0.726 & 0.693 & 0.552 & 0.701 && 0.731 \\
&& M-7B && 0.891 & 0.888 & 0.858 & 0.731 & 0.858 && 0.886  \\
&& Q-14B && 0.861 & 0.847 & 0.849 & 0.701 & 0.909 && 0.848  \\
&& D-8B && 0.779 & 0.785 & 0.765 & 0.664 & 0.848  && 0.792  \\
\specialrule{.1em}{.075em}{.075em}
\end{tabular}
}
\caption{Complete AUROC performance for baseline uncertainty algorithms and our methods on selected datasets. L-7B, L-13B, F-7B, M-7B, Q-14B and D-8B represent LLaMA2-7B, LLaMA2-13B, Falcon-7B, Mistral-7B-v0.1, Qwen2.5-14B and DeepSeek-R1-Distill-LLaMA-8B, which controll comparisons of 1) parameter-efficient architectures vs. standard transformers, 2) sparse vs. dense activation patterns, and 3) native training vs. distillation approaches. The methods are as described in \Cref{baseline}.
\label{tab:4}
}
\vspace{-3mm}
\end{table*}

This section presents the complete experimental results. 

\textbf{Generalization Study.} \Cref{tab:4} summarizes the AUROC scores comparing our method with baseline uncertainty algorithms across four NLG tasks and six LLM variants, demonstrating three key comparisons: 1) parameter-efficient architectures (e.g., DeepSeek-R1-Distill-LLaMA-8B) versus standard transformers, 2) sparse (e.g., Falcon-7B) versus dense activation patterns, and 3) native training versus distillation approaches.

Our method achieves competitive or superior performance in all model-dataset combinations. Notably, Shapley method outperforms all baselines on SQuAD (e.g., +2.6\% over discrete semantic entropy for LLaMA2-13B) and BioASQ (e.g., +3.4\% over semantic entropy for Mistral-7B), suggesting strong generalization to both general-domain and biomedical QA tasks. The distillation-based DeepSeek-R1 model shows particular compatibility with our method, attaining state-of-the-art AUROC on SQuAD.

Two noteworthy observations emerge: First, while embedding regression performs exceptionally well on SVAMP with Qwen2.5-14B, Shapley maintains robustness across all models for this arithmetic reasoning task. Second, the gap between Shapley and baselines widens for smaller models (e.g., +7.9\% over naive entropy for Falcon-7B on NQ-Open), indicating our method's effectiveness in low-resource scenarios. These results validate the importance of model-agnostic uncertainty quantification.

\label{app:hyparam}

\textbf{Kernel Hyperparameter.} We select the hyperparameter $\beta$ of the kernel function in our method by finding maximum mean AUROC through two rounds of systematic grid search. In first round, we initially constrain the search space to the original $\beta\in (0,1]$ with 0.1 increments. Based on the result of the first round, the search range is narrowed to $\beta\in [0.4, 0.6]$ with 0.02 increments. As shown in \Cref{tab:5}, the default hyperparameter $\beta$ should be set to 0.5.

\begin{table*}[t!]
\centering
\vspace{0mm}
\renewcommand\arraystretch{1.35}
\scalebox{0.8}{
\begin{tabular}{ccccccccccc}
\specialrule{.1em}{.075em}{.075em} 
$\beta$ & 0.1 & 0.2 & 0.3 & 0.4 & 0.5 & 0.6 & 0.7 & 0.8 & 0.9 & 1.0 \\
AUROC & 0.7886 & 0.7888 & 0.7888 & 0.7889 & \textbf{0.7891} & 0.7887 & 0.7883 & 0.7879 & 0.7874 & 0.7871 \\
\cline{1-11}
$\beta$ & 0.42 & 0.44 & 0.46 & 0.48 & 0.5 & 0.52 & 0.54 & 0.56 & 0.58 & 0.6 \\
AUROC & 0.78892 & 0.78894 & 0.78899 & 0.78905 & \textbf{0.78910} & 0.78902 & 0.78897 & 0.78891 & 0.78880 & 0.78866 \\
\specialrule{.1em}{.075em}{.075em}
\end{tabular}
}
\caption{AUROC performance for kernels with different $\beta$ values. LLaMA2-7B and TriviaQA are utilized as backbone dataset and LLM.
}
\label{tab:5}
\vspace{-3mm}
\end{table*}


\end{document}